\RequirePackage{fix-cm}

\documentclass[smallcondensed]{svjour3}     

\smartqed  
\usepackage{mathptmx}      
\usepackage{color}

\usepackage[left=1.0in,right=1.0in,top=1in,bottom=1in]{geometry}
\usepackage{graphicx}
\usepackage{amsmath}
\usepackage{amssymb}
\usepackage{amsfonts} 
\usepackage{amsxtra}
\usepackage{tikz-cd}
\usepackage{appendix}
\usepackage{upgreek}
\usepackage[colorlinks=true,linkcolor=blue]{hyperref} 
\usepackage{tabularx} 
\usepackage{wasysym} 
\usepackage{mathtools}
\usepackage{enumerate}
\usepackage{textcomp}
\usepackage{mathrsfs}
\usepackage{dsfont}
\usepackage[section]{placeins}
\usepackage[most]{tcolorbox}

\newtcolorbox{mytextbox}[1][]{%
  sharp corners,
  enhanced,
  colback=white,
  height=10cm,
  attach title to upper,
  #1
}
\usepackage{multirow} 
\usepackage{isomath} 
\usepackage{txfonts} 
\usepackage{tensor} 
\usepackage{pifont}
\usepackage{bm} 
\usepackage{framed}
\usepackage{placeins}
\usepackage{ulem}

\newenvironment{frcseries}{\fontfamily{frc}\selectfont}{}

\title{An Introduction to Kernel and Operator Learning Methods for Homogenization by Self-consistent Clustering Analysis}

\author{Owen Huang\textsuperscript {1}{\dag} \and
        Sourav Saha\textsuperscript {2} {\dag} \and
        Jiachen Guo\textsuperscript {2} {\dag} \and
        Wing Kam Liu\textsuperscript {2,3}{\ddag}
}
\institute{
    {\textsuperscript{1} Department of Mathematics, Princeton University, Washington Road, Princeton, 08540, New Jersey, USA.\\
    \textsuperscript{2} Theoretical and Applied Mechanics, Northwestern University, 2145 Sheridan Rd, Evanston, 60208, Illinois, USA.\\
    \textsuperscript{3} Department of Mechanical Engineering, Northwestern University, 2145 Sheridan Rd, Evanston, 60208, Illinois, USA.\\
    \textsuperscript{\dag} These authors contributed equally\\
    \textsuperscript{\ddag}Corresponding Author: Wing Kam Liu, 
	\email{w-liu@northwestern.edu}           
}
}

\begin{document}
	
\maketitle
\begin{abstract}
Recent advances in operator learning theory have improved our knowledge about learning maps between infinite dimensional spaces. However, for large-scale engineering problems such as concurrent multiscale simulation for mechanical properties, the training cost for the current operator learning methods is very high. The article presents a thorough analysis on the mathematical underpinnings of the operator learning paradigm and proposes a kernel learning method that maps between function spaces. We first provide a survey of modern kernel and operator learning theory, as well as discuss recent results and open problems. From there, the article presents an algorithm to how we can analytically approximate the piecewise constant functions on $\mathbb
{R}$ for operator learning. This implies the potential feasibility of success of neural operators on clustered functions. Finally, a $k$-means clustered domain on the basis of a mechanistic response is considered and the Lippmann-Schwinger equation for micro-mechanical homogenization is solved. The article briefly discusses the mathematics of previous kernel learning methods and some preliminary results with those methods. The proposed kernel operator learning method uses graph kernel networks to come up with a mechanistic reduced order method for multiscale homogenization.  
\end{abstract}

\keywords{Operator Learning \and Discrete Calculus \and Kernel Methods \and Lippmann Schwinger Equation \and Functional Analysis}

\section{Introduction}

The theory of neural networks have been studied extensively in the past century, and have been proven to be effective in learning highly nonlinear maps between finite-dimensional Euclidean spaces. Perhaps one of its most remarkable properties is the exhibition of a universal approximation theorem, which can be stated loosely as that any continuous function can be approximated to order $\epsilon$ with bounded depth. Equipped with this, and physical constraints, they have been used to efficiently solve many differential equations which show up in the natural sciences (ex. Physics-informed neural networks (PINN's) \cite{https://doi.org/10.48550/arxiv.2102.11802},\cite{RAISSI2019686}) Nonetheless, they are computationally and data expensive, and sensitive to the input and output domains, forcing a sense of rigidity in training.  This implies a fundamental disagreement with powerful macroscopic analysis techniques in materials engineering. The latter makes use of clustering techniques to capture the heterogeneous nature of microstructures. An example of a particular application of neural networks is solving instances of partial differential equations (PDE's). The subclass of problems when the underlying PDE is known has been studied extensively \cite{doi:10.1073/pnas.1718942115} and existing methods are able to solve small-scale problems efficiently (ex. constitutive laws \cite{DUBOS2020109629,li2019clustering}). Consider, for example, neural networks used for computational fluid dynamics (\cite{10.1145/2939672.2939738}), aerodynamic flow (\cite{Bhatnagar_2019}) and surrogate modelling (\cite{Zhu_2018}). Such methods require sufficient background knowledge regarding the governing physical laws, and are, by definition, discretization-dependent. However, in many scenarios, we do not have the liberty for such knowledge. We may not have enough prior knowledge to identify a closed form PDE (ex. many biological systems \cite{DBLP:journals/corr/abs-2108-08481}), and even if we can identify a nonlinear governing PDE, most mechanistic methods are too expensive when applied in the context of real engineering application. An appropriate example would be modeling heterogeneous materials with multiscale simulation. The established methods are either hierarchical \cite{vu2015multiscale,he2020hierarchical} or concurrent modeling \cite{gao2022concurrent,han2020efficient,yu2019multiresolution}. A significant volume of work have been done on multiresolution mechanics \cite{mcveigh2006multiresolution,mcveigh2008linking,mcveigh2009multiresolution,tian2010multiresolution}. To capture very critical aspects of the physics, strain or stress gradient methods are also proposed. All these methods have to make a compromise at some level regarding how much physics can be resolved and how many degrees of freedom are to be solved. More accurate models suffer from the need to solve a large number of degrees of freedom.   

To address these computational issues, reduced order models (ROM) are used frequently in engineering and scientific application. Some of the most famous ROMs are Proper Orthogonal Decomposition (POD) \cite{liberge2010reduced}, Proper Generalized Decomposition (PGD) \cite{dumon2011proper,lu2020adaptive,li2022stochastic}, and Transformation field analysis \cite{dvorak1992transformation}. This article will focus on a specific reduced order method called Self-consistent Clustering Analysis (SCA) \cite{liu2016self,yu2019self}. More details on the method will come later in the section 4.1. The idea of this method is to group the material point of a physical domain based some mechanical response and create a database of constitutive laws acting among these clusters. Later, based on the cluster deformation, the homogenized mechanical responses are computed in response to a macroscopic response. Such methods can drastically reduce the computational efforts of solving a PDE. However, some of these methods may lack generality or may require repetitive generation of the offline database.      

A relatively novel method to remedy this issue is generalization of reduced order models through neural networks or \textit{neural operators}. This device aims to learn the underlying operator between infinite dimensional Banach spaces $\Phi: \mathcal{A} \to  \mathcal{B}$, instead of just an instance of it, as is done with neural networks. In particular, we are interested in when $\mathcal{A}, \mathcal{B}$ are spaces of functions defined on a compact set $\Omega \subset \mathbb{R}^d$. This opens up entire new families of problems to be solved. The literature on neural operator theory is plentiful - certain architectures have already been shown experimentally to be orders of magnitude faster than conventional PDE solvers \cite{https://doi.org/10.48550/arxiv.2207.05209}. An example, which we will explore more in this paper is the Fourier Neural Operator (FNO's). This method introduces a "Fourier term" which leverages the convolution theorem to learn parameters in Fourier space. \cite{DBLP:journals/corr/abs-2010-08895}. Immediately, this is promising, because in the discrete setting, we can use the Fast Fourier Transform (FFT). Even more remarkably, it has been shown that FNO's possess a universal approximation theorem analagous to the well known theorem for neural networks \cite{https://doi.org/10.48550/arxiv.2107.07562}. As well, it enjoys invariance to discretization of domain. Unfortunately, the class of operators for which FNO's can universally approximate is too restrictive to be used alongside clustering methods used in engineering. It also turns out to be extremely data intensive during the training phase \cite{li2020fourier}. Another relevant architecture is the Graph Kernel Network (GKN), first introduced in \cite{DBLP:journals/corr/abs-2003-03485}. It is inspired by the Greens function method of solving differential equations, and thus includes an integral operator which is further parametrized by a shallow neural network. The GKN generalizes the FNO by decomposing the domain into a graph, then using a Monte Carlo sum to approximate an integral operator. Just like the FNO, the GKN is also faced with the curse of dimensionality as for large domains the training parameters of the graph neural network increases disproportionately making it less feasible for engineering application. However, these methods shows that the kernel to solve the PDEs ubiquitous in engineering and applied science can be \textit{learned} using deep neural network. The current research builds upon the premise that this \textit{kernel learning} approach can be extended to learn reduced order method such as SCA which will alleviate many of the current limitations of the method. 

This paper has three main objectives - the first is to offer a gentle but comprehensive plunge to the theory given working knowledge of introductory machine learning theory and real analysis. As we believe that operator learning theory is a fruitful domain which is still largely uncharted, it will be valuable to write in such a manner that experts and novices alike can access this paper. We also include all the relevant definitions and basic results as a crash course for future research. The second objective is an analysis of the most recent universal approximation theorem for FNO's, and how they relate to clustering methods. We will see that while they are not compatible, we can analytically represent clustering using differentiable functions and use those representations with FNO's. However, such an approximation may not be feasible or useful to implement, hence we search for a different approach. The paper's third goal is such an approach - an application of GKN's in materials sciences. SCA induces a natural graph structure on the domain, and we exploit this in the graph-construction phase of the GKN. We verify that this is consistent with the mathematical theory that underpins GKN's and then present numerical results in one dimension. 

\section{Kernel Theory}
In this section, we give a brief introduction to kernel learning, in particular motivating them, and the rigor behind the kernel trick. We will see that kernel learning is in some sense, ubiquitous in operator learning. 
\subsection{Motivation}
Kernel methods have been widely used in numerical methods. Among them, many are based on meshfree methods. Smoothed particle hydrodynamics (SPH) was first developed for hydrodynamics problems to solve astrophysical problems\cite{gingold1977smoothed}. The key idea of SPH is using kernel integration to approximate the field functions\cite{li2002meshfree}, as shown in Eq. \ref{SPH1}
\begin{equation}
 \langle f(x)\rangle=\int_{\Omega} f\left(x^{\prime}\right) \Phi\left(x-x^{\prime}, h\right) d x^{\prime} 
\label{SPH1}
\end{equation}
where $f(x)$ is the field function, $\Phi(x)$ is the so-called smoothing kernel function. Since $\Phi(x)$ is generally not the Dirac function, Eq. \ref{SPH1} is just an approximation of $f(x)$. In SPH, the continuous representation in Eq. \ref{SPH1} can be further approximated using discretized particles.
\\\\
Since SPH can have amplitude and phase errors and is deficient near the boundaries, Liu et. al \cite{liu1995reproducing,chen1996reproducing} proposed the reproducing kernel particle method (RKPM) to meet the completeness requirement by introducing a correction term in the kernel function. 
\begin{equation}
 \bar{\Phi}_{a}(x ; x-s)=C(x ; x-s) \Phi_{a}(x-s) 
\end{equation}
where $C(x ; x-s)$ is the correction term and can be expressed by a linear combination of polynomial basis functions.\\\\
Similarly, the kernel concept has been used to solve problems with discontinuity such as fractures. The peridyanmics method uses the integral formulation of continuum mechanics theory since classical derivative-based theory cannot handle crack surfaces and other forms of singularities\cite{silling2010peridynamic}. For example, the linear momentum balance can be written in the following integration form.
\begin{equation}
 \rho(x) \ddot{u}(x, t)=\int_{R} f\left(u\left(x^{\prime}, t\right)-u(x, t), x^{\prime}-x, x\right) d V_{x^{\prime}}+b(x, t) 
\end{equation}
where $t$ is time, $u$ is the displacement vector field, $\rho$  is the mass density in the undeformed body, $x'$ is a dummy variable of integration and $f$ is the pairwise force density (kernel function) that $x'$ exerts on $x$. \\

A more rudimentary but enlightening motivating example arises from regression. Consider the solution to an ordinary least squares (OLS) regression problem, with data set $\{(x_i,y_i)\}_{i=1}^N$ where $x_i \in \mathbb{R}^n$ and $y_i\in \mathbb{R}$ and optimal solution $y = (\beta^\star)^T x$. Recall that this has a closed form solution - namely $\beta^\star = (x^Tx)^{-1}x^Ty$. When we would like to upgrade the linear regression, to incorporate more linearities, one can instead introduce basis functions $\phi_i(x)$ instead (sometimes, these are also called features). Nonetheless, because the regression is still linear with respect to the features, the OLS solutions still hold, and we can obtain the optimal solution $\beta^\star = (\Phi^T\Phi)^{-1}\Phi^Ty$ where $\Phi$ is the matrix $(\Phi)_{ij} = \varphi_j(x_i)$. In reality, computing this matrix is very expensive, especially when the $\phi_j$ are extremely nonlinear. Fortunately, we can rewrite the optimal solution in another more versatile form - namely $\beta^\star = \Phi^T(\Phi\Phi^T)^{-1}y$. At first glance, this is almost identical - we still need to compute the $\Phi$ matrix, but we can actually exploit the properties of $K \coloneqq \Phi\Phi^T$ to bypass this. We only need to note that $K$ is symmetric and positive semidefinite, for it turns out that this is enough to use an important theorem to drastically reduce computation time. 
\subsection{Kernels, Mercer's Theorem and the Kernel Trick}
\begin{definition}
A function $K: \mathcal{X}\times \mathcal{X}\to \mathbb
{R}$ is a \textbf{kernel function} on a set $\mathcal{X}$ if it is symmetric: $K(x,y) = K(y,x)$ and positive semidefinite: 
\begin{equation}\label{eq:ker}
\sum_{j=1}^n\sum_{i=1}^na_ia_jK(x_i,x_j) \geq 0
\end{equation}
for every $a_1,\dots,a_n \in \mathbb{R}$ and $x_1,\dots,x_n \in \Omega$. 
\end{definition}
We remark that if the domain of $K$ satisfies $\mathcal{X} \subset \mathbb{R}^d$ compact, then (\ref{eq:ker}) simply becomes 
\begin{equation}
    \int_{\mathcal{X}\times \mathcal{X}}K(x,y)\varphi(x)\varphi_(y)dxdy \geq 0
\end{equation}
Kernels are generalizations of positive-semidefinite matrices from linear algebra. That being said, it is not surpising that every positive semidefinite matrix is a kernel! Indeed, (positive) linear combinations of kernels yields another kernel, as do products. Some nontrivial examples include the Gaussian kernel: $K(x,y)=\exp (\|x-y\|/2\sigma^2)$ and the Laplacian kernel: $K(x,y) = \exp(-\alpha\|x-y\|)$. Other examples include inner products on Hilbert spaces, and more importantly in deep learning theory, the Reproducing Kernel Hilbert Space (RKHS). Equipped with the definition of a kernel, we have from a theorem of Mercer \cite{mercer1909xvi} that kernels are rather easy to axiomatize. 
\begin{theorem}[Mercer]
    Suppose that $K: \mathcal{X}\times \mathcal{X}\to \mathbb
{R}$ is a kernel function, where $\mathcal{X} \subset \mathbb{R}^d$ compact. Then there is an orthonormal basis of $L^2(\mathcal{X}
)$, say $\{\varphi_i\}_{i=1}^\infty$ with nonnegative eigenvalues $\lambda_i$ so that 
\begin{equation}
    K(x,y) = \sum_{i=1}^\infty\lambda_i\varphi_i(x)\varphi_i(y)
\end{equation}
and the sum converges uniformly and absolutely. 
\end{theorem}
and we have the following as a corollary,
\begin{corollary}
    $K:\mathcal{X}\times \mathcal{X}\to \mathbb
{R}$ is a kernel function if and only if $K$ can be represented as an inner product $K(x,y)= \langle\varphi(x),\varphi(y)\rangle_{\mathcal{H}}$ for some $\varphi$ and Hilbert space $\mathcal{H}$.
\end{corollary}
If we interpret the $\varphi_i$ as learnable features, one might interpret the matrix $\Phi$ as a map into some feature space. Then we can think of the above theorems as the statement that \textit{every kernel corresponds to a feature mapping}. Moreover, applying this directly to $\Phi\Phi^T$ matrix in section 2.1, we can rewrite
\begin{equation}
    \Phi\Phi^T = \begin{bmatrix}
\varphi(x_1)^T\varphi(x_1) & \hdots & \varphi(x_1)^T\varphi(x_N) \\
\vdots & \ddots & \hdots \\
\varphi(x_N)^T\varphi(x_1) & \vdots & \varphi(x_N)^T\varphi(x_N)
\end{bmatrix}=\begin{bmatrix}
K(x_1,x_1) & \hdots & K(x_1,x_N)\\
\vdots & \ddots & \hdots \\
K(x_N,x_1) & \vdots & K(x_N,x_N)
\end{bmatrix} 
\end{equation}
for some kernel function $K$. This is called the kernel trick, and allows us to bypass computing a highly nonlinear $\varphi$ by passing to a kernel function. We remark that $K$ need not be unique! It turns out that the idea of swapping a (presumably sophisticated) inner product for a kernel to be learned via deep learning is the basis for \textit{kernel learning theory}. However, in the present literature, some \textit{kernels} need not satisfy Definition 1 - they just keep the namesake. Consider for example, a numerical method for computing Greens functions. In machine learning, the notion of introducing a kernel functions is ubiquitous, as we can exploit its algebraic properties to reduce computational energy. Indeed, they are used extensively for testing independence between random variables \cite{smola2007hilbert}, dimensionality reduction \cite{rahimi2007random} and even genomic data analysis \cite{wang2015kernel}. In the next chapter, we explore of its applications in operator learning. 
\section{Neural Operator Theory}
In the previous section, we surveyed basically what is the mathematical motivation for the major neural operator architectures we discuss next. In this section, we will develop the mathematical foundation of neural operators. All of the preliminaries needed for this section will be found in the appendix. Then, we will motivate and construct the GKN, from which it will be obvious that the FNO is a special case. To understand what we mean by learning operators, we'll first have to rigorously formulate what neural networks really do. Here, we'll follow the characterization of neural networks adapted from \cite{roberts2021principles}.

\subsection{Basic Theory of and Constructions of Neural Operators}
Traditionally, neural networks seek to learn a mapping between finite-dimensional spaces:  $\Phi:\mathbb{R}^n\to \mathbb{R}^k$. Given i.i.d. samples (with respect to some probability measure) $\{x_i,y_i\}_{i=1}^N$ with $x_i \in \mathbb{R}^n$ and $y_i = \Phi(x_i)$ (possibly noisy observations), we aim to construct a parametrized function $\Phi':\mathbb{R}^n \times \Theta \to \mathbb{R}^k$ such that for some $\theta^\star \in \Theta \subseteq \mathbb
{R}^p$, 
\begin{equation*}
    \Phi'(x, \theta^\star) \approx \Phi(x)
\end{equation*}
More formally, 
\begin{definition}
A \textbf{Neural Network} $\mathcal{N}$ is a function $\mathcal{N}:\mathbb{R}^n \to \mathbb{R}^k$ which admits a decomposition
\begin{equation}
    \mathcal{N} = Q \circ \mathcal{L}_L\circ \mathcal{L}_{L-1} \circ \dots \circ \mathcal{L}_1 \circ P
\end{equation}
with $P$ a $d \times n$ matrix called the \textbf{encoder} and $Q$ a $k \times d$ matrix called the \textbf{decoder}. We'll note that sometimes, this is called an \textit{encoder-decoder NN}. We chose this definition in order to demonstrate similarites with neural operators. Indeed, to get the more generic neural network form, one can simply take $P$ and $Q$ to be matrices with 1's on the diagonal. Each $\mathcal{L}_j: \mathbb{R}^{d_{j-1}}\to \mathbb{R}^{d_j}$ is a \textbf{layer} is of the form 
\begin{equation}
    \mathcal{L}_j(v) = \sigma(W_jv+b_jv)
\end{equation}
for a $d_j \times d_{j-1}$ matrix $W_j$, scalar $b_j$ and \textbf{activation function} $\sigma: \mathbb{R}\to \mathbb{R}$ (taken to be non-affine and usually non-polynomial). There are $L$ layers.
\end{definition}
The remarkable result concerning neural networks may be summarized in the \textit{Universal Approximation Theorem for Neural Networks} where the most modern formulation is as follows \cite{DBLP:journals/corr/abs-1905-08539}:
\begin{theorem}[Kidger, Lyons 2020]
Let $f: \Omega \to \mathbb{R}^k$ be in $C^0(\Omega; \mathbb{R}^k)$ for a compact subset $\Omega \subset \mathbb{R}^n$. Let $\mathcal{N}^\sigma(n,k,L)$ denote the space of all neural networks with inputs in $\mathbb{R}^n$, outputs in $\mathbb{R}^k$ with an arbitrary number of hidden layers, each with $L$ neurons and non-affine activation function $\sigma:\mathbb{R}\to \mathbb{R}$. Then for every $\epsilon>0$, there exists a element $f^\star \in \mathcal{N}^\sigma(n,k,n+k+2)$ such that 
\begin{equation}
    \sup_{x \in \Omega}\Vert f(x)-f^\star(x)\Vert < \epsilon
\end{equation}
\end{theorem}
We remark that while the theorem affirms the existence of a "good" neural network, actually finding the network is often intractable. Moreover, such a function need not be efficient. Indeed, it is the case that generally, a very accurate neural network is \textit{not} fast! Consider a traditional image classification problem solved by a Convolutional Neural Network (CNN). The solution space is intimately tied to the input resolution. Any attempt to resize the input image would require some sort of graphical preprocessing (and the classifier still need not work). A network which has prediction error independent of input resolution is called \textit{discretization-invariant}, and we will explore this more rigorously in this section. \\\\
On what seems to be a completely unrelated objective, that of efficiently learning PDE operators, we explore a very natural idea to generalize neural networks to learn mappings between infinite dimensional spaces. Remarkably, it turns out that this method actually gives us discretization-invariance!
More formally, consider an operator between Banach spaces $\Phi:\mathcal{A} \to \mathcal{B}$ (usually function spaces). Given i.i.d samples $\{f_i,g_i\}_{i=1}^N$, with $g(i) = \Phi(f_i)$ possibly corrupted with noise and $\epsilon>0$, we search for an operator $\Phi^\star: \mathcal{A} \to \mathcal{B}$ such that for a compact subset $\Omega \subset \mathcal{A}$, 
\begin{equation}
    \sup_{f\in \Omega}\Vert\Phi f-\Phi^\star f\Vert_\mathcal{K}< \epsilon
\end{equation}
For sake of concreteness, take $\mathcal{A} = H^s(D; \mathbb{R}^n)$ and $\mathcal{B} = H^s(D; \mathbb{R}^k)$ for a compact subset (usually a rectangular region) $D \subset \mathbb{R}^u$. \\\\
A \textbf{neural operator} parallels neural networks for learning operators. Formally, a neural operator is a map $\mathcal{N}$ from $H^s(D; \mathbb{R}^n)$ to $H^s(D; \mathbb{R}^k)$ of the form: 
\begin{equation}\label{eq:a2}
    \mathcal{N} = \mathcal{Q} \circ \mathcal{L}_L\circ \mathcal{L}_{L-1} \circ \dots \circ \mathcal{L}_1 \circ \mathcal{P}
\end{equation}
with the \textbf{lifting} linear operator $\mathcal{P}: H^s(D; \mathbb{R}^n) \to H^s(D; \mathbb{R}^d)$ be given by 
\begin{equation}
    \mathcal{P}[f](x) = Pf(x)
\end{equation}
for a $d \times n$ matrix $P$. $\mathcal{Q}: H^s(D; \mathbb{R}^d) \to H^s(D; \mathbb{R}^k)$ is the \textbf{projection} operator, and is given by 
\begin{equation}
    \mathcal{Q}[f](x) = Qf(x)
\end{equation} 
for a $k \times d $ matrix $Q$. The layers $\mathcal{L}_j: H^s(D; \mathbb{R}^{d_{j-1}}) \to H^s(D; \mathbb{R}^{d_j})$ comprise of weights and biases and activation functions, just as their neural network counterparts. However, there is often added structure, such as the addition of an integral operator.\\\\
We would like to highlight two particular major properties that neural operators enjoy: discretization-invariance and universal approximation. Let's make that rigorous, inspired by the framework in \cite{DBLP:journals/corr/abs-2108-08481}.
\begin{definition} 
Let $D \subset \mathbb{R}^u$ be compact. Let $\{D_i\}_{i=1}^\infty$ be a nested sequence of sets so that each $D_j$ contains $j$ points. Such a sequence is a \textbf{refinement} of $D$ if for every $\epsilon >0$, there is an $L$ so that any point in $D$ is no more than $\epsilon$ away from some point of $D_L$ (using Euclidean distance). Symbolically, this is equivalent to that we can find a $D_L$ for which 
\begin{equation}
    D \subseteq \bigcup_{x\in D_L}B(x,\epsilon)
\end{equation}
For any refinement, we call each $D_j$ a $j$-\textbf{point discretization} of $D$. 
\end{definition}
Now, with an $j$ - point discretization and a function $f:D\to \mathbb{R}^n$, define $(D_j,f(D_j) \coloneqq \{d,f(d)\}_{d\in D_j}$ be the pointwise evalutations of the function $f$ on the points in the discretization. Equipped with this characterization, we would like to define what it means for an learned operator to be discretization invariant. 
\begin{definition} 
    Let $\Phi: H^s(D; \mathbb{R}^n)\to H^s(D; \mathbb{R}^d)$ be an operator and suppose that $D$ admits a refinement$\{D_j\}_{j=1}^\infty$. We say that $\Phi$ is \textbf{discretization-invariant} if we can find maps $\widetilde{\Phi}_j:\mathbb{R}^n
    \times \mathbb{R}^d \to H^s(D; \mathbb{R}^d)$ for any compact $K \subset H^s(D; \mathbb{R}^n) $, 
    \begin{equation}\label{eq:991}
        \lim_{j\to \infty} \sup_{f\in K} \|\widetilde{\Phi}_j(D_j, f(D_j)) - \Phi f\|_{H^s} = 0 
    \end{equation}
\end{definition}
We remark that we can alter this definition to include parametrized operators $\Phi$, but for the sake of highlighting relevant topics, this is omitted. This is an attractive property to enjoy, as we will see numerically, because it allows for flexibility in choosing mesh size/geometry as well as theoretically reducing computation time (ie. using a low-resolution model to use in high-resolution computations). Fortunately, certain neural operator architectures satisfy discretization-invariance, the statement and proof of which we will defer to \cite{DBLP:journals/corr/abs-2108-08481}. 
It is also remarkable that a universal approximation theorem similar to Theorem 1 can be achieved. Put more formally, we see from \cite{392253}
\begin{theorem}[Chen,Chen 1995]\label{eq:991}
Let $D \subset \mathbb{R}^n$ be compact and $K \subset C^0(D, \mathbb{R})$ be compact too. Suppose that $V \subset \mathbb{R}^d$ is compact and we have the continuous operator $\Phi: K \to C^0(V, \mathbb{R})$. If $\sigma$ is not a polynomial (ex. $\textrm{ReLU}$), then for every $\epsilon>0$, there exists a wide one layer neural operator $\Phi^\star$ so that for every $f \in K$
\begin{equation}
    | \Phi f-\Phi^\star f| < \epsilon
\end{equation}
\end{theorem}
It seems as though that neural operators are very useful, and enjoy a lot of the properties that neural networks do not. Unfortunately, it turns out learning the operator (say, a linear differential operator) is much harder than learning the solution for a particular instance of it. Nonetheless, because of its invariance under geometry of domain, neural operators become much more cost-effective in multiple-evaluation problems. In the rest of this section, we will explore various popular neural operator architectures, which will be relevant in the later sections. 
\subsection{Graph Kernel Networks}

One rudimentary neural operator architecture is called the Graph Kernel Network (GKN), and is an architecture inspired by the construction of the Green's function approximation calculation (see: Appendix). It was first introduced in \cite{https://doi.org/10.48550/arxiv.2003.03485}. It relies on graph-structured data, which has already seen fair utility in CNN's \cite{https://doi.org/10.48550/arxiv.1609.02907} as well as in molecule modelling \cite{journalarticle}. Let us construct them, as well as explore some of the properties they enjoy. Using the same notation as in \eqref{eq:a2}, a GKN is a neural operator where each layer $\mathcal{L}_j: H^s(D; \mathbb{R}^d) \to H^s(D; \mathbb{R}^d)$ is given by 
\begin{equation}
    \mathcal{L}_jf = \sigma\left(W_jf+b_jf+K(a; \theta)f\right)
\end{equation}
where the integral operator $K(a,\theta):H^s(D; \mathbb{R}^d) \to H^s(D; \mathbb{R}^d)$ is given by
\begin{equation}\label{eq:a10}
    [K(a;\theta)f](x) \coloneqq  \int_D \kappa_\theta(x,y,a(x),a(y))f(y)dy
\end{equation}
where the integral kernel $\kappa_\theta
:\mathbb{R}^{2d+2n} \to \mathbb{R}^{d\times d}$ is parametrized by $\theta \in \Theta$, a finite-dimensional parameter space. Here, $\kappa_\theta$ is itself modeled as a neural network defined on a graph constructed on the domain $D$. Its specific architecture can vary depending on utility. In light of kernel methods, we remark that the theme of learning a kernel via a NN is present. We also remark that the integral may be computed over any Borel measure, but we will restrict to the Lebesgue measure in this paper and computations. 
To compute the integral in (\ref{eq:a10}), the domain $D$ 
is partitioned into the complete graph $K_n$, then a message-passing graph network is used. More specifically, equation \eqref{eq:a10} is approximated by a Monte Carlo sum
\begin{equation}\label{eq:99}
    v_{t+1}(x)= \sigma \left(Wv_t(x) + \frac{1}{\lvert N(x)\rvert}\sum_{y\in N(x)}\kappa_\theta(e(x,y))v_t(y)\right) 
\end{equation}
where $v_t(x) \in \mathbb{R}^{n_v}$ are \textit{node features} and $e(x,y) \in \mathbb{R}^{n_e}$ are \textit{edge features}. Figure 6 is a visualization of a domain decomposition. We remark that this is in fact equivalent to a \textit{message passing graph network} \cite{gilmer2017neural}. We relate this to SCA (described in the next section) by letting $v_t(x) = \beta_I$ for $x \in D_I$ and $e_t(x,y) = D_{IJ}$ for $x \in D_I$ and $y \in D_J$. Here, $\kappa_\theta: \mathbb{R}^{n_e} \to \mathbb{R}^{n_v\times n_v}$ is a variable (usually shallow) inner neural network. Using the graph structure, we are able to capture many non-local properties which a uniform mesh or finite-element triangulation cannot measure. Traditional domain-decomposition algorithms which compile the within-cluster average behaviour of materials do not consider nonlocal cluster interactions. This approach is thus specimen-specific and hence fails when generalizing to highly nonlinear behaviours, such as strain and elasticity. In the view of computational accuracy, the decomposition of the domain to $K_n$ with appropriate node and edge weights is important to capture all nonlocal effects.  Unfortunately, it is apparent that representing a domain as the complete graph is computationally expensive, as the number of edges increases on the order of $O(|V(G)|^2)$. One typical method employed to reduce the computation cost is Nyström approximation:  \\\\
Consider the set of nodes in the full graph $V(G) = \{v_1,v_2,\dots,v_n\}$ and for a fixed $m\ll n$, choose a subset of $m$ vertices $V' = \{v_{j_1},v_{j_2},\dots,v_{j_m}\}$ uniformly at random and construct a random subgraph induced on those vertices. We then approximate the sum part of (\ref{eq:99}) by 
\begin{equation*}
    \int_D \kappa_\theta(x,y,a(x),a(y))f(y)dy \approx \frac{1}{\lvert V'\rvert}\sum_{y\in V'}\kappa_\theta(e(x,y))v_t(y)
\end{equation*}
indeed, if we repeat this process $s$ times, the computation complexity grows no faster than $O(sm^2) \ll O(|V(G)|^2)$. Finally, we remark that discretization-invariance of GKN's have not yet been analytically established, but numerical results from section 5 as well as \cite{https://doi.org/10.48550/arxiv.2003.03485}\cite{li2020multipole} show that test error is consistent through varying resolution, confirming hypotheses. In the following subsection, we impose a clever restriction on $\kappa_\theta(x,y,a(x),a(y))$ so that we may use Fourier analysis, and further, yield a universal approximation theorem which is stronger than Theorem 3. Nonetheless, we conclude this subsection with an open problem.

\begin{tcolorbox}[colback=white!5!white,colframe=black!75!black]
\begin{problem}
    Prove or disprove that Graph Kernel Networks, when viewed as an operator between function spaces, are discretization-invariant, in the sense of (\ref{eq:991}).
\end{problem}
\end{tcolorbox}

\subsection{Fourier Neural Operators}
The Fourier neural operator was first introduced in \cite{li2021fourier} (see Fig. \ref{fig:FNO}), and exploits the convolution operation to allow for much of the learning to be done in Fourier space. Moreover, in conjunction with discretization algorithms, one can use the fast Fourier transform (FFT) as well, allowing for even more efficient computations. A \textbf{Fourier neural operator} is a kernel network with integral kernel $k_\theta$ void of any dependence on the function $a$. That is, of the form 
\begin{equation}
\kappa_\theta(x,y,a(x),a(y)) \coloneqq \kappa_\theta(x-y)
\end{equation}
for some parametrized $\kappa_\theta$ which makes (\ref{eq:a10}) into a convolution operator:
\begin{equation}
    [K(a;\theta)f](x) =  \int_D \kappa_\theta(x-y)f(y)dy
\end{equation}
Applying the convolution theorem, we see obtain
\begin{equation}
    K(a;\theta)f = \mathcal{F}^{-1}(\mathcal{F}(\kappa_\theta)\cdot\mathcal{F}(f))
\end{equation}
and so we may put everything together to derive the layer structure for the FNO:
\begin{equation}\label{eq:fno1}
    \mathcal{L}_jf = \sigma\left(W_jf+b_jf+\mathcal{F}^{-1}(\mathcal{F}(\kappa_\theta)\cdot\mathcal{F}(f))\right)
\end{equation}
for every $1\leq j \leq L$. FNO's have shown to be extremely effective in solving PDE's - the results presented from \cite{li2020fourier} are orders of magnitudes faster than traditional solvers. Figure 1 gives a diagram of the FNO architecture. It can be seen FNO has 3 main building blocks, namely encoder $P$, Fourier blocks, and decoder $Q$. The encoder $P$ can lift the dimension of the input features and therefore uncover hidden information in the original input features. The Fourier block follows (24) . Finally, the decoder $Q$ is used to reconstruct the solution in the original domain. Recently, it has been shown that with more Fourier blocks and a larger number of Fourier modes, it is expected that FNO can approximate any nonlinear operators\cite{https://doi.org/10.48550/arxiv.2107.07562}.
\begin{theorem}[Kovachki et al., 2017]\label{eq:a11}
\label{thm:thm1}
Let $s \geq 0$, and $\Phi: H^s(D; \mathbb{R}^n) \to H^s(D; \mathbb{R}^k)$ be a continuous operator, and $\Omega \subset H^s(D; \mathbb{R}^n)$ be compact. For every $\epsilon > 0$, there exists a Fourier neural network $\mathcal{N}:H^s(D; \mathbb{R}^n) \to H^s(D; \mathbb{R}^k)$ (which can be regarded as a continuous operator) such that 
\begin{equation}
    \sup_{a \in \Omega}\Vert\Phi a - \mathcal{N}a\Vert_{H_s} < \epsilon
\end{equation}
\end{theorem}

\begin{figure}[ht]
\centering
\includegraphics[scale=0.65]{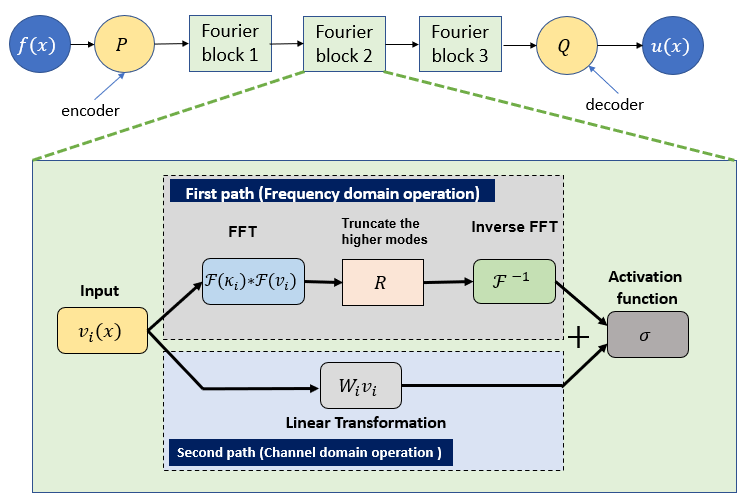}
\caption{Visual representation of an FNO. Each Fourier layer performs computations in Fourier space, then returns with an FFT from which normal neural network mechanisms take place.}
\label{fig:FNO}
\end{figure}
One may remark on the similarities to Theorems 2 and 3, as the only differences really are only the domains and codomains of the learned mappings. The figure above is visual on the layers of a FNO used in practice. One begins with a FFT into the Fourier space, where we learn the kernel $\kappa$. However, in practice we see that it is useful to truncate frequencies above a certain threshold, both to eliminate potential noise and reduce computation cost. Applying the inverse FFT, summing a linear term and then post-composing with an activation function is precisely the architecture described in (\ref{eq:fno1}).
\section{Multiscale Homogenization by Self-consistent Clustering Analysis (SCA)}

\subsection{Self-consistent Clustering Analysis (SCA)}

Hierarchical material systems are widespread in nature. A myriad of concurrent simulation techniques has been developed to model the mechanical behavior of hierarchical multiscale materials, such as FE-FE  \cite{feyel2000fe2}, FE-FFT  \cite{kochmann2016two}. However, one common drawback of these methods is the exorbitant computational cost, especially when the degrees of freedom (DOF) of the system are large. To this end, various reduced-order modeling techniques have been applied to multiscale simulation problems to reduce the computational cost \cite{michel2003nonuniform,yvonnet2007reduced,ladeveze2010latin,liu2016self}. Among them, Self-consistent clustering analysis (SCA) is a mechanistic, data-driven method that can model the mechanical response of heterogeneous materials while effectively saving computational time using the $k$-means clustering algorithm..

The computational paradigm of SCA mainly consists of 2 stages, namely the offline stage and the online stage. In the offline stage, material points with similar mechanical behaviors are grouped together using $k$-means clustering technique. The material points are clustered based on the strain concentration tensor $\mathbf{A}_{m}(\mathbf{x})$, which is defined as
\begin{equation}
 \boldsymbol{\varepsilon}_{m}(\mathbf{x})=\mathbf{A}_{m}(\mathbf{x}): \boldsymbol{\varepsilon}_{M} 
\end{equation}
where $\boldsymbol{\varepsilon}_{M}$ is the prescribed macroscopic strain along different directions, $\boldsymbol{\varepsilon}_{m}$ is the resultant local strain in the microscale level. Note only elastic response is required for the offline computations. With the obtained strain concentration tensor $\mathbf{A}_{m}(\mathbf{x})$ at each different material point, the whole computational domain can be effectively decomposed by the $k$-means clustering algorithm. Therefore, the number of DOF for the problem can be significantly reduced. Furthermore, the interaction tensors $ \mathbf{D}^{I J}$ among different clusters can be computed.

\begin{equation} \label{Eq 27}
    \bm{D}^{IJ} = \frac{1}{c^{I}|\Omega|} \int_{\Omega} \int_{\Omega}\chi^{I}(x)\chi^{J}(y) \mathcal{G}(x,y)dxdy.
\end{equation}

In the equation, $ \mathbf{D}^{I J}$ is the interaction tensor between $I$th and $J$th clusters, $c$ is cluster identifier, $I, J$ are the cluster indices, and $\chi^I(x)$ is the cluster variables which assumes value 1 when $x$ is in cluster $I$ and zero otherwise, and $\mathcal{G}(x,y)$ is the Green's function. In the online stage, the mechanical response of the microstructure is obtained by solving the incremental form of discretized \textbf{L}ippmann-\textbf{S}chwinger (LS) equation.

\begin{equation}
 \Delta \varepsilon^{I}+\sum_{J=1}^{k} \mathbf{D}^{I J}:\left[\Delta \sigma^{J}-\mathbf{C}^{0}: \Delta \varepsilon^{J}\right]-\Delta \varepsilon^{0}=0 
 \label{lp}
\end{equation}
where $\mathbf\Delta \varepsilon^{J}$ and $\mathbf\Delta \sigma^{J}$ are incremental strain and stress in the $J$th cluster; $k$ is the number of clusters; $\mathbf{C}^{0}$ is the reference material stiffness; $\mathbf\Delta \varepsilon^{0}$ is the far-field strain. Note that Eq \ref{lp} is obtained by averaging the stress and strain in each different cluster. Since the number of clusters is far less than the number of DOF of the original system, it can be much faster to solve Eq \ref{lp}. For more details please see \cite{liu2016self}.

\subsection{Solving Clustered LS Equation using Neural Operators}

\subsubsection{Fourier Neural Operators and Clustering}
Theorem \ref{thm:thm1} is useful, in the sense that we now know at least a subset of all the functions FNO's can approximate. Unfortunately,  we cannot immediately use FNO's on clustered functions, in particular those which are the result of SCA. This follows simply because step functions are not $H^s$. However, we can design an algorithm to produce a good approximation of piece-wise constant functions on $\mathbb{R}$ for which we can apply FNO's. Firstly, recall that for the Heaviside step function 
\begin{equation*}
     H(x) = \left\{
\begin{array}{ll}
      0 & x <0  \\
      \alpha & x \geq 0\\
\end{array} 
\right.
\end{equation*}
we have the analytic approximation 
\begin{equation}
    \lim_{k\to \infty}\left(\frac{\alpha}{2}
    +\frac{\alpha}{2}\tanh kx\right) = H(x)
\end{equation}
holding pointwise and distributionally for every $x$. \\\\
Let $D = [a,b]$ be an interval domain, and suppose we are given some domain decomposition (ex. SCA) $D = D_1 \sqcup \dots \sqcup D_m$ where each $D_j$ are of the form $[d_{j-1},d_j)$, and let $\Gamma$ be the set of all possible such finite partitions. Let $\varphi
: D \to \mathbb{
R}$ be of the form
\begin{equation}\label{eq:a6}
         \varphi(x) = \left\{
\begin{array}{ll}
      \omega_1  &, x \in D_1 \\
      \omega_2 &, x \in D_2 \\
      \vdots & \vdots \\
      \omega_k &, x \in D_m \\
\end{array} 
\right.
\end{equation}
for constants $\omega_1,\dots,\omega_k$. Equivalently, functions of the form in (\ref{eq:a6}) may be written in step function form as 

\begin{equation}\label{eq:a7}
\varphi(x) = \sum_{j=1}^k\omega_{j}\mathds{1}_{D_j}(x)
\end{equation}

where $\mathds{1}$ is the indicator function. Our goal is to approximate $\varphi$ with a $H^s$, smooth almost-everywhere function $f$. Luckily, there is a natural extension of the approximation of the Heaviside function to functions of the form (\ref{eq:a7}). We do so iteratively: 
Write $\Omega_j = [a_{j-1}, a_j]$ so that each $a_j$ is a \textit{break point} and $a_0 = a$ and $a_k = b$. Let $y_j = (a_{j-1}+a_j)/2$ be the midpoints of each $\Omega_j$. Let $f_1: [a,y_2] \times \mathbb{N}\to \mathbb{R}$ be given by 
\begin{equation*}
    f_1(x,k) = \omega_1+\frac{\omega_2-\omega_1}{2}
    +\frac{\omega_2-\omega_1}{2}\tanh k(x-a_1)
\end{equation*}
Notice this is just an affine transformation of the step-approximation. Now, let $f^*_2: [y_2,y_3] \times \mathbb{N} \to  \mathbb{R}$ be given by 
\begin{equation*}
    f^*_2(x,k) = \omega_2+\frac{\omega_3-\omega_2}{2}
    +\frac{\omega_3-\omega_2}{2}\tanh k(x-a_2)
\end{equation*}
To define $f_2$, we need to first define a correction function at each midpoint $\epsilon: \{1,\dots,k\} \to \mathbb{R}$ by $\epsilon(1) = 0$ and $\epsilon(2) = \omega_2-f_1(y_2)$. Next, define 
\begin{equation*}
    f_2(x) = f_2^*(x) + \epsilon(2)
\end{equation*} 
so that we force agreement at midpoints (and hence produce an $H^s$ function): $f_1(y_2) = f_2(y_2)$. Now, we extend this construction inductively for the rest of the clusters in the obvious way: \\\\

Given that $f_1, \dots, f_{j-1}$ and $\epsilon(1), \dots, \epsilon(j-1)$ are already defined, we construct $f_j^*: [y_j, y_{j+1}] \times \mathbb{N} \to \mathbb{R}$ by 
    \begin{equation*}
    f^*_j(x,k) = \omega_j+\frac{\omega_{j+1}-\omega_j}{2}
    +\frac{\omega_{j+1}-\omega_j}{2}\tanh k(x-a_2)
\end{equation*}
and $\epsilon(j) = \omega_j - f_{j-1}(y_j)$. Finally, let $f_j(x) = f_j^*(x)+\epsilon(x)$ for $2\leq j \leq m-1$. To ensure the domain $[a,b]$ is covered, just extend the final $f_{m-1}: [y_{m-1}, y_m] \to \mathbb{R}$  continuously to $[y_{m-1}, b]$. Finally, define the continuous, differentiable almost everywhere function $f_k: D\to \mathbb{R}$ by 
\begin{equation}
         f_k(x) = \left\{
\begin{array}{ll}
      f_1(x, k)  & x \in [a,y_2] \\
      f_j(x, k) & x \in [y_j,y_{j+1}] \\
      f_{m-1}(x, k) & x \in [y_{m-1},b] \\
\end{array} 
\right.
\end{equation}
for every $n \in \mathbb{N}$. In fact, this approximation function defines an operator $\mathcal{A}_{\gamma,k}: L(D) \to H^s(D)$, where $L(D)$ is the vector space of piecewise constant functions on $D$ equipped with the supremum norm and $\gamma \in \Gamma$. It is not hard to see that
\begin{theorem} Let $D = [0,1]$ and fix a domain decomposition $\gamma \in \Gamma$. The operator $\mathcal{A}_{\gamma,k}: L(D)\to H^s(D)$ taking $\varphi \mapsto \mathcal{A}_{\gamma,k} \varphi = f_k$ is a well-defined continuous operator. Moreover, 
\begin{equation}
    \lim_{k \to \infty}\lvert \varphi(x) - \mathcal{A}_{\gamma,k}\varphi(x))\rvert = 0
\end{equation}
pointwise almost everywhere.
\end{theorem}

In general, current clustering techniques such as $k$-means and SCA are not compatible with the most recent universal approximation theorem for FNO's. This follows simply because piecewise constant function are not Sobolev. Theorem 5 implies that this problem is theoretically resolved using the $\mathcal{A}_{\gamma,k}$ operator in the following way: \\\\
Let $\Phi: H^s(D) \to H^s(D)$ be an operator between Sobolev spaces which we are trying to learn. Fixing a domain decomposition $\gamma$, let $C_{\gamma}: H^s(D) \to L(D)$ be a clustering operator for which $C_\gamma\varphi$ is a piecewise constant approximation of the input function $\varphi$. However, $\varphi \not \in H^s$, so we cannot apply FNO's immediately. Now using $\mathcal{A}_{\gamma,k}$ as a "good" approximation of the piecewise constant function $C_\gamma \varphi$, which is in turn a good approximation of the $H^s$ function $\varphi$, then we may be interested in learning the operator $\mathcal{A}_{\gamma,k} \circ \mathcal{C}_\gamma:H^s(D) \to H^s(D)$. Of course, outside of a theoretical perspective, this should not be evidently any more efficient than using an FNO to learn $\Phi$ directly. This is summarized in the following (non commutative) diagram \\
\begin{tikzcd}
\centering
    H^s(D) \arrow[d, "C_{\gamma}"]\arrow[r,"\Phi"] & H^s(D)\\L(D) \arrow[ur, "\mathcal{A}_{\gamma,k}"']
\end{tikzcd}\\
With this, we present two open questions. The first concerns the fact that most mechanistic settings require domains to be in two or three dimensions. Problem 2 asks for construction of an analytic approximation in the sense of Theorem 5. The second problem concerns the practicality of such approximations. Although the question is limited to FNO's, it is nonetheless a technical curiousity.  

\begin{tcolorbox}[colback=white!5!white,colframe=black!75!black]
\begin{problem}[High-dimensional analogue]
Let $D \subset \mathbb{R}^d$. Let $\varphi: D \to \mathbb{R}$ be piecewise constant on a  partition $\gamma$. Does there exist a well-defined continuous approximation operator $\mathcal{A}_{\gamma,k}: L(D) \to H^s(D)$ for which 
\begin{equation}
    \lim_{k \to \infty}\lvert \varphi(x) - \mathcal{A}_{\gamma,k}\varphi(x))\rvert = 0
\end{equation}
pointwise almost everywhere?
\end{problem}
\end{tcolorbox}

\begin{tcolorbox}[colback=white!5!white,colframe=black!75!black]
\begin{problem}[Efficacy of Theorem 5]
Assuming the affirmative to Problem 2, is learning an operator $\Phi$ with a full-field method faster or the clustered-approximated operator $\mathcal{A}_{\gamma,k} \circ \mathcal{C}_\gamma$ faster? 
\end{problem}
\end{tcolorbox}
Nonetheless, while we cannot justify the efficacy of FNO's at this moment, we can still achieve significant results with graph kernel networks (see Section 3.2) The following section is an exhibit of experimental results of FNO's and GKN's used in conjunction with SCA. 
\section{Numerical Examples}
This section provides numerical example of solving LS equation with kernel learning of the operator. In the previous sections it was shown that the FNO may be feasible for spaces of clustered functions and three problems are proposed for general operator learning. The current part of the article will give a small working example for FNO working on clustered space and show numerical evidence that graph kernel learning also works for one-dimensional functional space.

\subsection{Modeling Clustered Lippmann-Schwinger equation using FNO}

\begin{figure}[hbt!]
\centering
\includegraphics[scale=0.7]{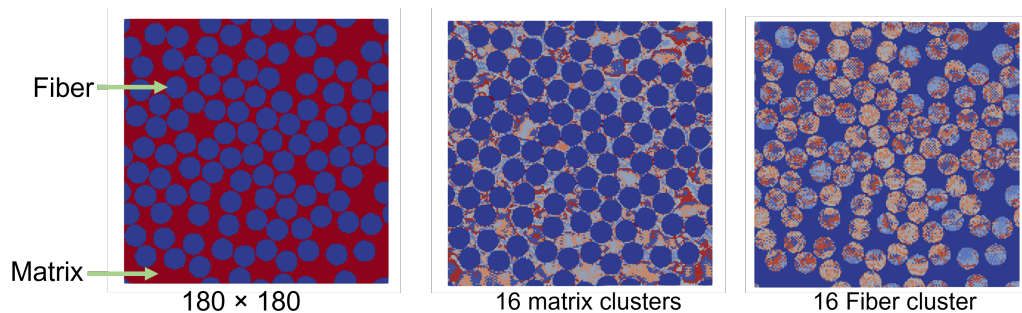}
\caption{An example two-phase composite material for solving Lippmann-Schwinger (LS) equation by SCA and FNO. Each matrix and fiber phases are divided into 16 clusters.}
\end{figure}



\begin{figure}[h]
\centering
\includegraphics[width=0.75\textwidth]{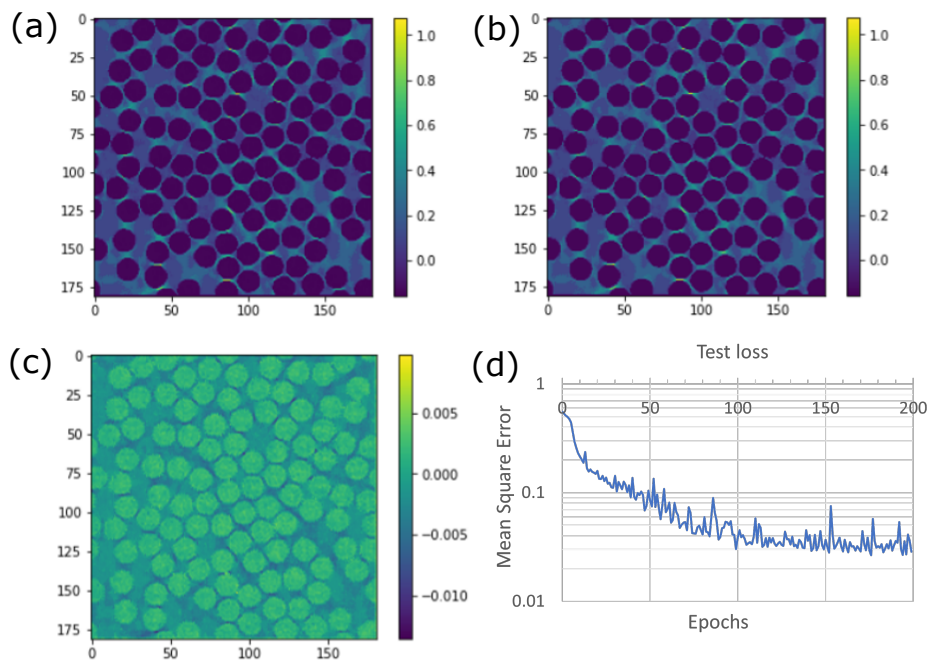}
\caption{Numerical prediction of local strain distribution along X-axis for solving the LS equation, (a) FNO solution, (b) SCA solution, (c) Relative error of prediction, (d) Training performance. The solutions are presented for 16 clusters per phase.}
\label{fig:Fig_domain16}
\end{figure}

\begin{table}[h!]
\caption{Material constants for matrix and inclusion}
\label{table:material_constants}
\centering
\begin{tabular}{ccccc} 
 \hline
 $E_{\text{matrix}}$ (GPa) & $\nu_{\text{matrix}}$ & $E_{\text{inclusion}}$ (GPa) & $\nu_{\text{inclusion}}$ & Area fraction of inclusion \\ [0.5ex] 
  \hline
  3.79 & 0.39 & 245.0 & 0.28 & 0.12\\ [0.5ex]
 \hline
\end{tabular}
\end{table}
\begin{equation} \label{eq:yeilding}
f=\bar{\sigma}-\sigma_{Y,matrix}\left(\bar{\varepsilon}^{p}\right)\leq 0
\end{equation}

\begin{equation} \label{eq:harden}
    \sigma_{Y,matrix} = \begin{cases}
    0.50 + 5\bar{\varepsilon}^\textrm{p}, &0<\bar{\varepsilon}^\textrm{p}\le 0.04\\
    0.62 + 2\bar{\varepsilon}^\textrm{p}, &0.04<\bar{\varepsilon}^\textrm{p}
\end{cases}
\end{equation}

The example two-dimensional composite material considered for numerical analysis is shown in Figure 2. The material consists of matrix phase (red) and fiber phase (blue). Each phase is divided into 16 clusters for training. The input structure is of $180 \times 180$ grid points and it is reduced to 32 clusters in total. The matrix material is isotropic elasto-plastic while the fiber is orthotropic elastic. The material properties are mentioned in the Table \ref{table:material_constants}. For training the FNO, the 1500 training and 500 testing data was used from strain $0-0.2$. For simplicity only uniaxial tensile strain along $X$ is applied. The FNO has 5 hidden layers. Since the matrix is elasto-plastic, we tried to capture the sensitivity to input strain. The network was trained on microstructure and applied strain as inputs, and the total strain as the output following the construction of LS equation. The testing results for SCA computation and FNO are presented for 16 clusters are presented in Fig. \ref{fig:Fig_domain16}. In Fig. \ref{fig:Fig_domain16}(a) the FNO solution and in (b) the kernel learning results are presented. The local solution seems to converge well for 16 clusters (see the errors in Fig. \ref{fig:Fig_domain16}(c) and (d)). This proves the efficacy of the FNO method to approximate the solution of LS equation. However, when the trained network was used to predict the results for 32 clusters in the same simulation domain. The results are shown in Fig. \ref{fig:Fig_domain32}. We can observe that the error of approximation is at least one order of magnitude higher. However, with the analysis, the authors do not conclude that FNO cannot extrapolate the solution for number of clusters as the network was trained on 16 clusters and varying strain values. The purpose of presenting this example is to prove that it is possible to approximate the clustered domain solving the LS equation using FNO. This follows from our previous discussion in section 4.3.1. However, one feature of FNO is it uses convolutional kernel on all the material points in the Fourier domain. While the transformation to Fourier domain achieves the discretization independence, the training procedure involves a large number of parameters. Following this observation, we move on to propose a graph kernel-based learning method that can approximate the solution for the entire domain based on convolution operator on a few cluster centroids.


\begin{figure}[h]
\centering
\includegraphics[width=0.95\textwidth]{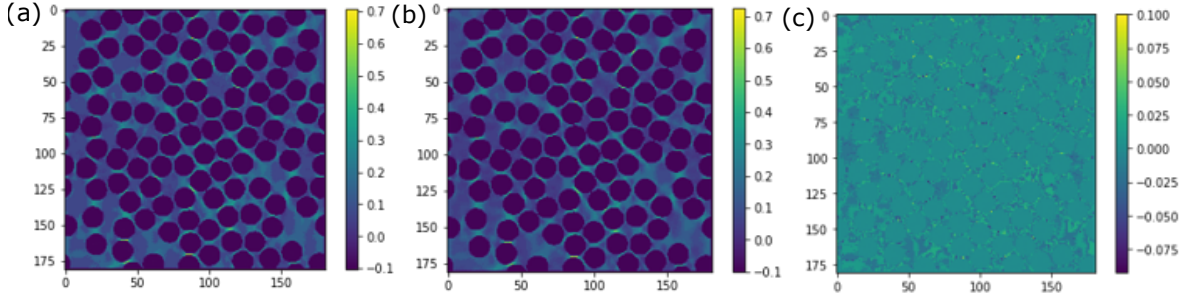}
\caption{Numerical prediction of local strain distribution along X-axis for solving the LS equation, (a) FNO solution, (b) SCA solution, (c) Relative error of prediction. The solutions are presented for 32 clusters per phase.}
\label{fig:Fig_domain32}
\end{figure}

\subsection{Modeling Clustered Lippmann-Schwinger equation using Graph Kernels}
This section thoroughly analyzes how graph kernel learning can be used to solve the LS equation. For convenience of understanding, a one-dimensional example is solved and the parameters are varied to mimic the engineering constraints. A brief summary on the method can also be found in the companion paper \cite{saha2023dldc} published in the same special issue.

\begin{figure}[h]
\centering
\includegraphics[width=0.85\textwidth]{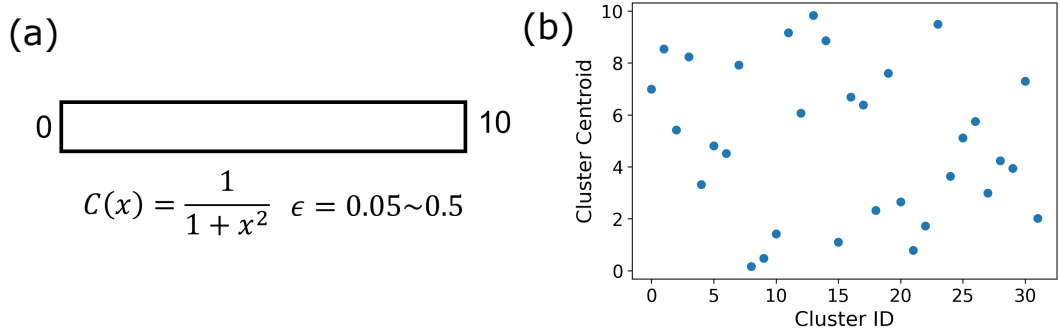}
\caption{(a) One-dimensional domain for solving the LS equation. (b) A sample of distribution of cluster centroids.}
\label{fig:Fig_domain}
\end{figure}

The sample one-dimensional domain is shown in Fig. \ref{fig:Fig_domain}(a). For training, the domain size was taken as $x$ = $0$ to $10$. The number of material points for solution was 16384. The material is considered to be linear elastic and the stiffness is $C(x)=\frac{1}{1+x^2}$. The strain was varied from 0.05 to 0.35 to generate the solution. Since this is a one-dimensional problem, the applied macroscopic strain is a scalar. The domain is then clustered using $k$-means clustering. The idea of kernel learning on cluster centroids is explained in Fig.\ref{fig:1D_cl}. The cluster centroids are taken as the nodes of the graphs. The far-field strain and co-ordinates of the cluster centroids are taken as the edge attributes. The radius of neighborhood construction $r$ follows Eq. \ref{eq:99} and accounts for how many cluster centroids are interacting with each other. 

\begin{figure}[ht]
\centering
\includegraphics[scale=0.75]{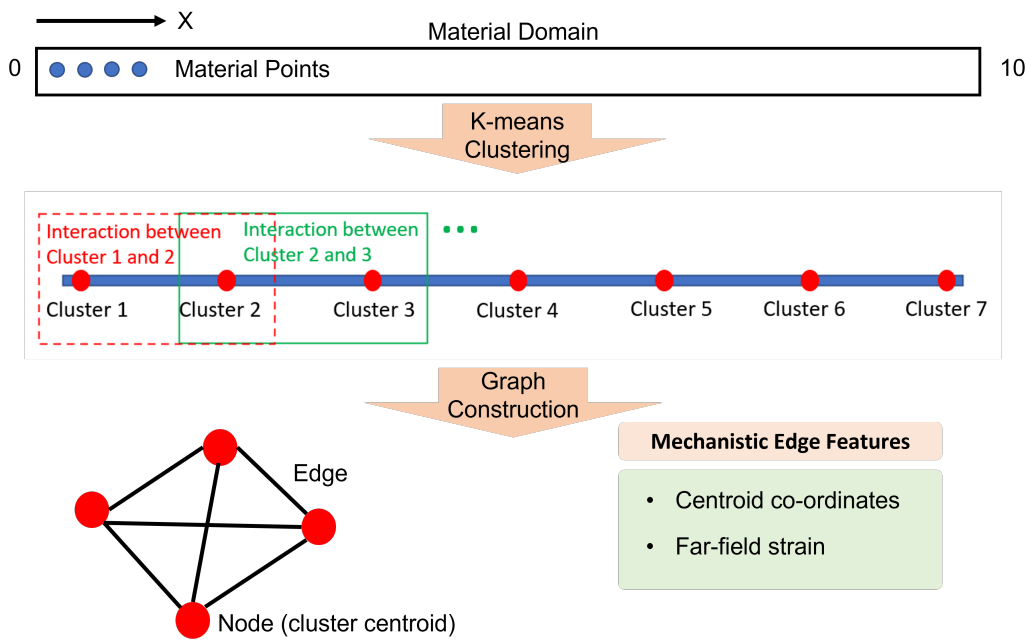}
\caption{The idea of domain reduction for graph kernel learning. The entire physical domain is clustered based on some physical response. The cluster centroids are taken as nodes of the graph and edge (and/or node) features include mechanistic parameters such as co-ordinates and strain.}
\label{fig:1D_cl}
\end{figure}

\begin{figure}[!]
\centering
\includegraphics[width=0.95\textwidth]{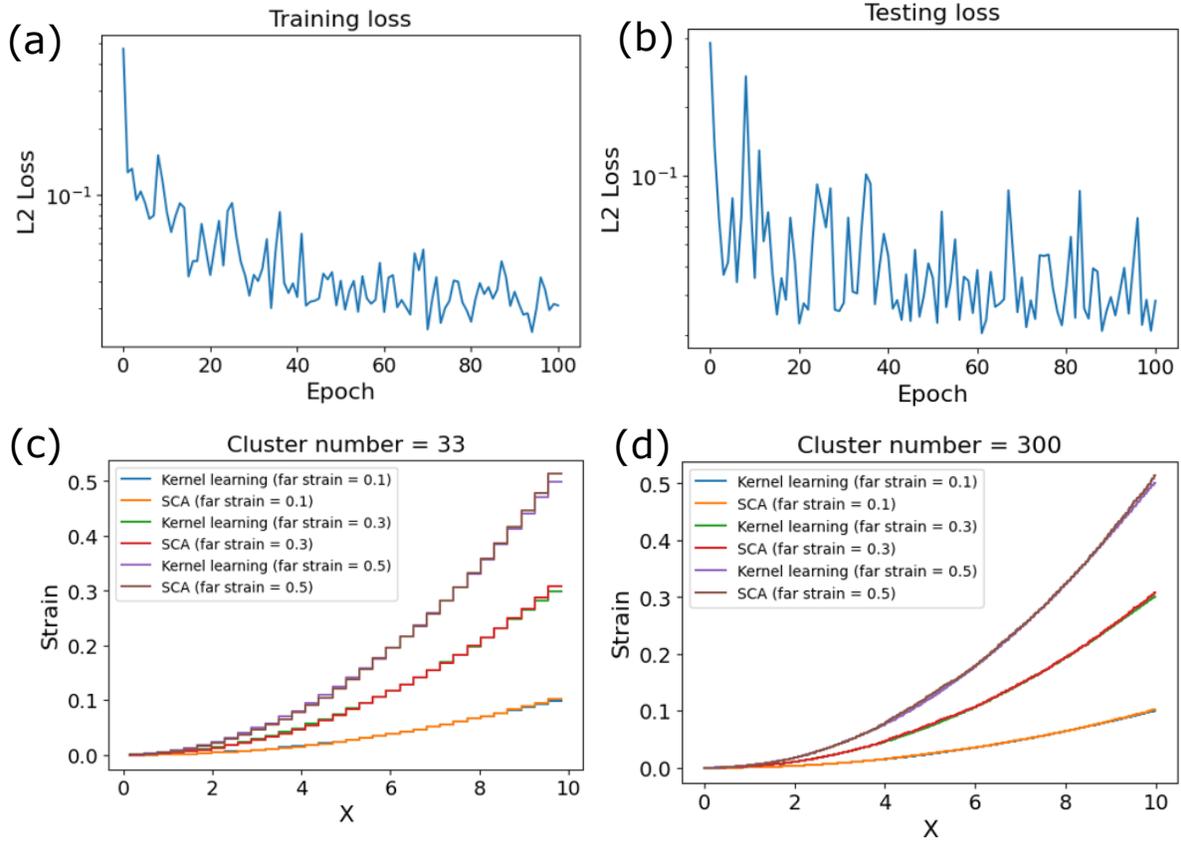}
\caption{(a) Training loss, and (b) testing loss as function of epochs. The results for strain levels 0.13, 0.26, and 0.39 for (c) 33 clusters and (d) 300 clusters.}
\label{fig:Fig_domain}
\end{figure}

The training data are generated by SCA solver by varying the far field strain and a set of clusters defined as $2^n$ where $n$ = 1, 2, 3, 4, 5, 6. The training performance is shown in Fig. \ref{fig:Fig_domain}(a) and (b) in terms of L2 loss of the training and testing process. Figs. \ref{fig:Fig_domain}(c) and (d) shows the performance of the trained kernel on predicting the local strain distribution. For each case, the applied far-field strains are 0.13, 0.26, and 0.39. The last strain value is outside the training range. Fig. \ref{fig:Fig_domain}(c) shows the results for 33 clusters in the domain. It is clear that the results converge very well with the SCA computation for all the strains. If the number of clusters is increased beyond the training range (such as 300 in Fig. \ref{fig:Fig_domain}(c)), the learned kernel can predict the local solution field as good as SCA. The results show the efficacy of the kernel learning method of the interaction matrix operator to be "discretization invariant". That means if in a multiscale simulation setting, the kernel is learned on a smaller number of clusters and can be used to predict the solution for a higher number of clusters.       

\begin{figure}[ht]
\centering
\includegraphics[scale=0.55]{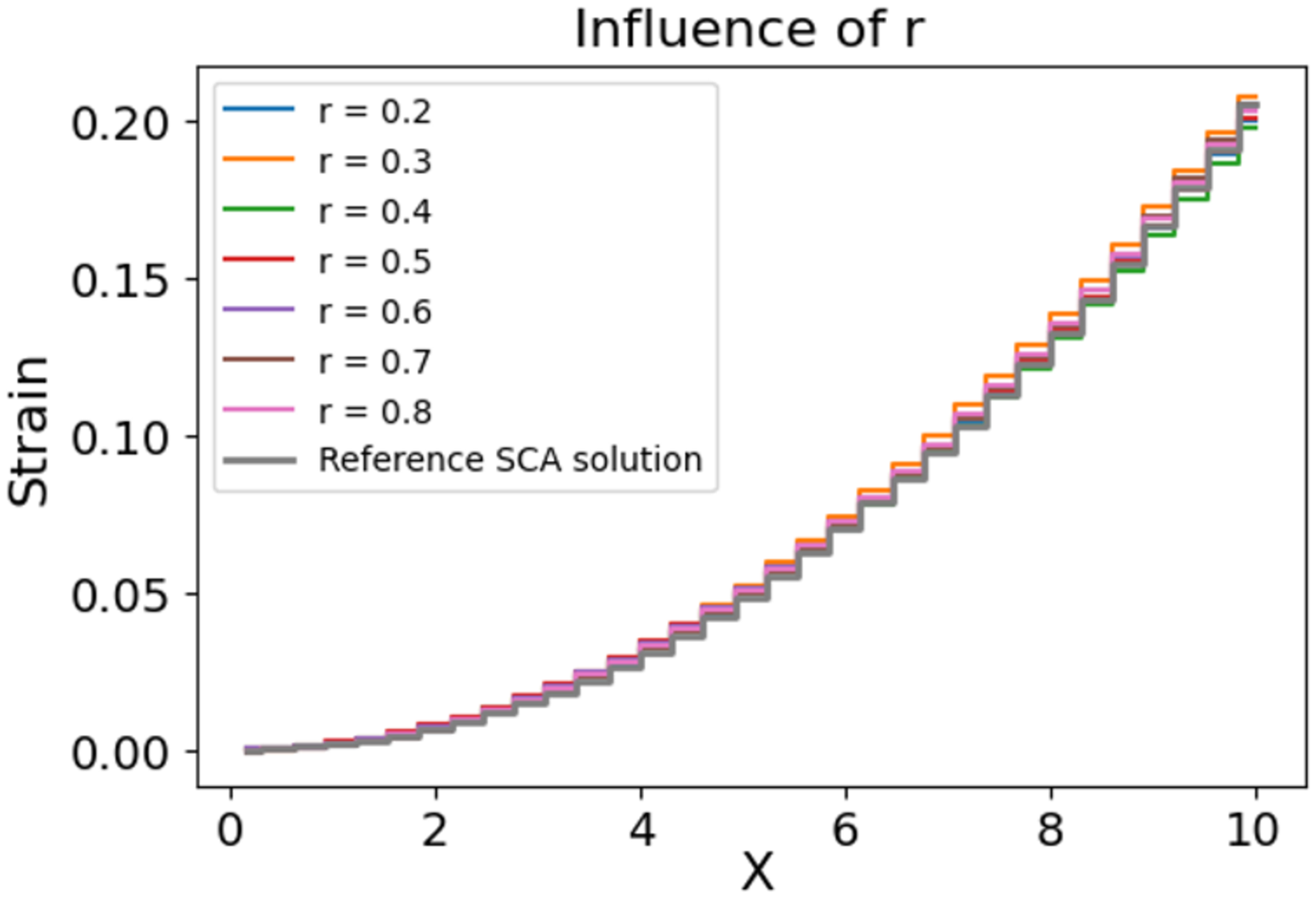}
\caption{Effect of adding more cluster centroids for creating the neighborhood. $r=1$ means all the centroids are considered.}
\label{fig:r_vary}
\end{figure}

\begin{table}[ht]
\centering
\begin{tabular}{l|llll}
\cline{1-2}
Radius & L2 loss &  &  &  \\ \cline{1-2}
2      & 0.0718  &  &  &  \\
3      & 0.0252  &  &  &  \\
4      & 0.0206  &  &  &  \\
5      & 0.0113  &  &  &  \\
6      & 0.0562  &  &  &  \\
7      & 0.0337  &  &  &  \\
8      & 0.0582  &  &  &  \\ \cline{1-2}
\end{tabular}
\caption{The L2 errors as a function of the radius of influence.}
\label{table2}
\end{table}

To understand the effect of how the radius of influence or the number of cluster centroids used for kernel learning on the local strain field approximation, the size of $r$ is varied and the resulting approximation error is observed for cluster 33. The results of this analysis are presented in Fig. \ref{fig:r_vary} and Table \ref{table2}. The value of $r$ is scaled to be below 1. When $r$ = 1, all the cluster centroids are taken into consideration. Interestingly, there is no specific trend in varying $r$ with accuracy. It seems that all the considered values give a reasonable accuracy. For $r$ = 0.5 the results are the closest to the reference solution. This analysis needs to be extended to further understand the effect of the influence of radius.     

\begin{figure}[h!]
\centering
\includegraphics[scale=0.8]{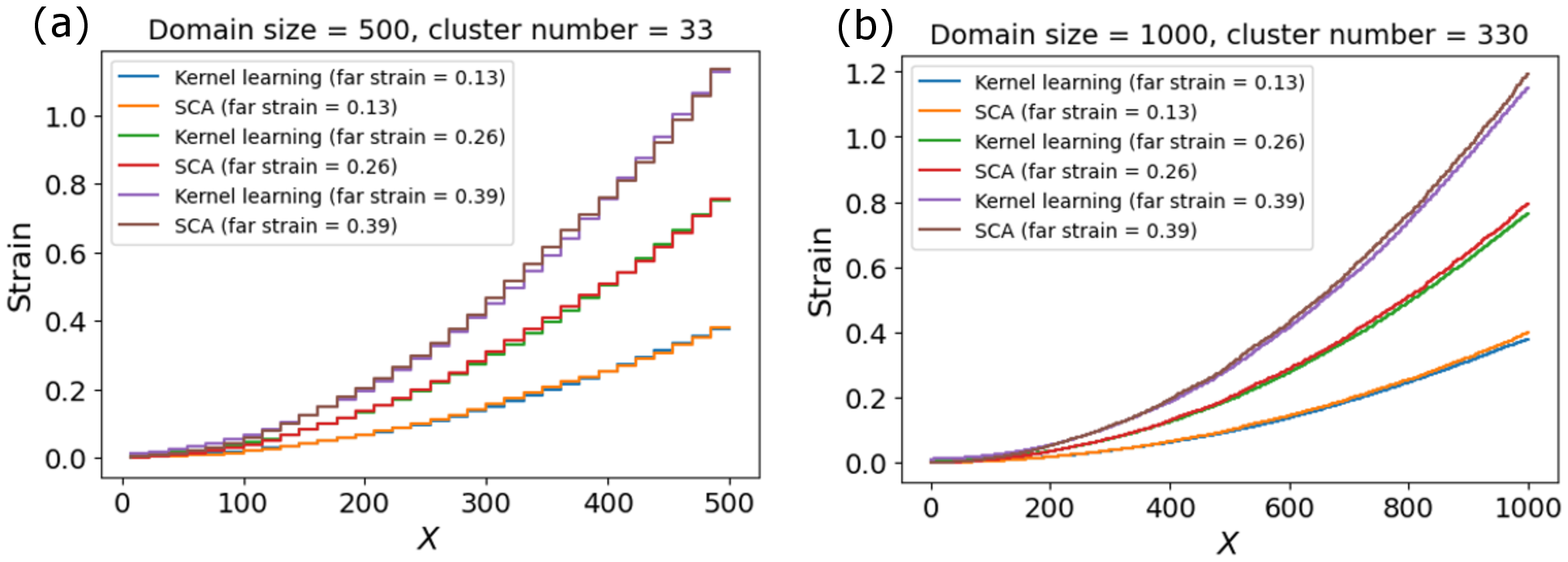}
\caption{Performance of the graph kernel learning in predicting an increased domain size of (a) 500 and (b) 1000. The original network was trained on a domain of size 10.}
\label{fig:domain_extension}
\end{figure}

Figure \ref{fig:domain_extension} shows perhaps the most important observation with the kernel learning approach. In this figure, the performance of the trained kernel is tested for a larger domain from 0 to 500 and to 1000. The domain used for training is 0 to 10. Hence, the predicted domain is 2-3 orders of magnitude larger. In both plots, the applied macro-strain is varied at three levels. The plot shows that the learned kernel performs as good as the SCA solution in approximating the solution. This is a significant progress on the SCA solution scheme. The result implies that kernel can learn the interaction tensor on a smaller training domain and can predict on larger domain. This alleviates a major bottleneck of SCA analysis as the offline computation requires expensive interaction tensor computation. Using kernel learning method, repetitive computation of the interaction tensor can be avoided. Here, an argument can be put forth that the training of the kernel still requires generation of data and computation of SCA solution including the interaction tensor. However, the finding suggests that if the kernel is trained on a smaller domain (RVE) of a hierarchical materials system, it can be used to predict solution of a larger domain such as solved for in \cite{saha2021microscale,saha2021macroscale}. The computation of interaction tensor for large number of degrees of freedom can be computationally more challenging compared to generating data in a smaller domain. 

Another major area of improvement in traditional SCA is adaptive clustering \cite{ferreira2022adaptivity}. Such technique is required for homogenization if there exists region of high strain gradient inside the RVE. The kernel learning approach can remove that circumvent the problem as well. As a demonstration of this capability, Fig. \ref{fig:Fig_adaptive} shows a case of composite clustering where the domain is divided into two sets of clusters. In Fig. \ref{fig:Fig_adaptive} (a), the first half of the domain is divided into 16 clusters (coarse) and last half is divided into 300 clusters (fine). For the Fig. \ref{fig:Fig_adaptive} (b) the same domain is divided into 32 and 600 clusters. For both cases, applied macrostrain was 0.2. The figure shows that the prediction for such composite clustering is as good as the SCA solution. This is an evidence that, with kernel learning, enrichment of the domain with additional clusters is possible. Extending this concept to 2D and 3D real microstructure could potential make multiscale materials modeling even faster.      

\begin{figure}
\centering
\includegraphics[width=\textwidth]{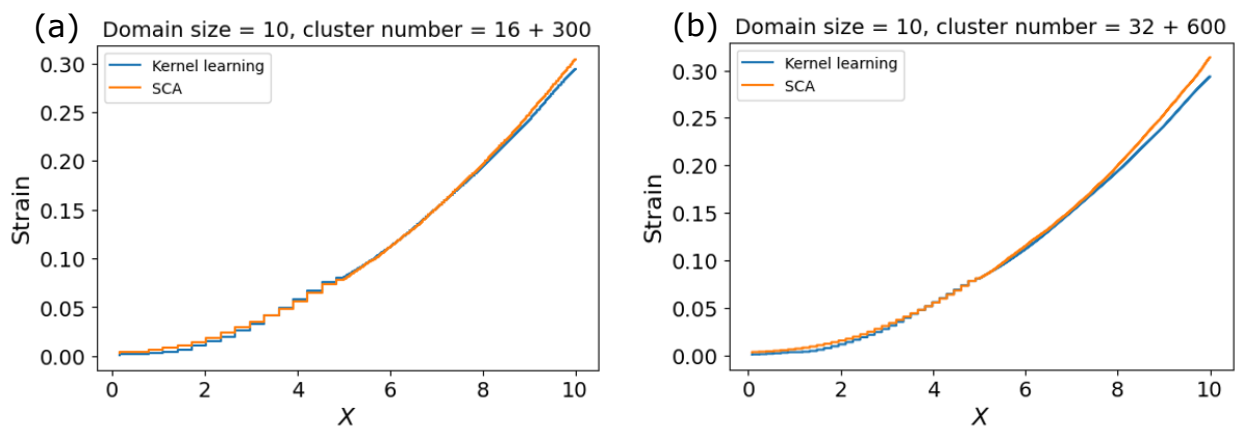}
\caption{Performance of kernel learning for composite clustering. (a) Domain size 10, 16 clusters in $0 \leq x \leq 5$, and 300 clusters in $5 < x \leq 10$. (b) Domain size 10, 32 clusters in $0 \leq x \leq 5$, and 600 clusters in $5 < x \leq 10$. }
\label{fig:Fig_adaptive}
\end{figure}
\section{Discussions and Future Work}

The article provides a thorough expository analysis of how operator and kernel learning have been implemented in recent literature. We have also examined the structures of two emerging neural operator architectures, GKN's and FNO's and asked three questions regarding their mathematical properties. The article also provides a new insight on how kernel learning can be used to learn the SCA computation and possibly extend the horizon of application of SCA. However, the examples presented are still in the preliminary phase. The demonstrative example of kernel learning is performed for linear elastic, one-dimensional problem. There is a chance that the parameters are overfitting. Hence, on the next step, two- and three-dimensional microstructures will be tested with non-linearity. The current one-dimensional analysis shows a potential to offset the computational cost of interaction tensor for multiple application of SCA method for same material and different microstructural instance. The adaptive clustering needs to be implemented for damage initiation and propagation cases. While modeling the high or low cycle fatigue, the atomistic shear strain becomes constant over a few cycles. Keeping this scope in mind, SCA was used to model such high cycle fatigue phenomenon in additive manufacturing alloys \cite{kafka2021image}. The trained kernel mimicking SCA can be used in this context where the trained kernel will be used to predict the evolution of atomistic shear strain.

\section{Conclusions}
The kernel learning algorithms presented in the article provide an insight on how to implement the operator learning in a reduced domain. The ROM presented in the paper has the potential to be useful for multiscale materials design. The capability of adaptive clustering and discretization invariance can offset the cost of computing the interaction tensor each time a new instantiation of microstructure is analyzed with SCA. On the other hand, the proposed method shows a way to use the graph-based operator learning for large-scale micromechanics problems where the traditional approach could result in a large number of training parameters. The article also provides a systematic analysis on how to use the kernel-based operator learning approaches for clustered domains and proposes three open problems to be proved. The numerical examples cover from demonstrations of the deep learning methods presented in the article and new development in graph-based learning. The discussions on the mathematical is deliberately kept brief albeit complete so that readers from applied mathematics and engineering can have easy access for introductory concepts.

\begin{acknowledgements}

The authors would like to acknowledge the support of National Science Foundation (NSF, USA) grants CMMI-1762035 and CMMI-1934367 and AFOSR, USA grant FA9550-18-1-0381. The authors would like to acknowledge
the contribution of Alberto Ciampaglia, Department of Mechanical and Aerospace Engineering, Politecnico di Torino, Italy and Visiting
Researcher, Department of Mechanical Engineering, Northwestern University to this article.

\end{acknowledgements}

\begin{appendices}
\section{Appendix 1: Preliminaries, Real and Functional Analysis}\label{secA1}
 In this section, we provide a self contained treatment of the machinery needed for the rest of the paper. We only assume working knowledge of undergraduate linear algebra, multivariable calculus and basic real analysis. The material will be adapted from \cite{Stein:1385521}, \cite{10.2307/j.ctvcm4hpw}.
\subsection{Glossary of Notation}
\begin{center}
\begin{tabular}{||c c ||} 
 \hline
 Symbol & Definition  \\ [0.5ex] 
 \hline\hline
 $\mathbb{R}^k$ & $k$-dimensional Euclidean space \\ 
 $\mathcal{A}(D,X)$ & A Banach space of functions from $D$ to $X$. \\
 $C^n$ & The space of all functions whose $n$'th derivative is continuous\\
 $L^p$ & The space of all functions with integrable $p$-powers\\
 $H^s(D,X)$ & The Sobolev space of all functions from $D$ to $X$\\
 $\Theta$ & A finite dimensional parameter space\\
 $\mathcal{F}(f)$ & The Fourier transform of an integrable function $f$\\
 $\sup_{\Omega}$ & The supremum over the set $\Omega$\\
 $K_n$ & The complete graph on $n$ vertices\\
 $B(x,r)$ & The ball of radius $r$ centered at $x$\\
 \hline
 
\end{tabular}
\end{center}
\subsection{Real and Functional Analysis}
We will briefly explore the basic results in real and functional analysis needed to motivate and support the theoretical portions of this paper. Suppose $U$ is an open subset of $\mathbb{R}^n$. Denote the complex vector space of all infinitely-differentiable functions $f: U \to \mathbb{C}$, with compact support to be $C_c^\infty(U)$. 
Recall that we may construct the dual space by taking the real vector space of all linear functionals $\varphi: C_c^\infty(U) \to \mathbb{C}$. 
\begin{definition}
A \textbf{distribution} on $\mathbb{R}^n$ (or any smooth manifold) is a continuous linear functional $u: C_c^\infty(\mathbb{R}^n) \to \mathbb{C}$. The space of all distributions will be denoted $\mathcal{D}'(\mathbb{R}^n)$. 
\end{definition}
where by continuity, we must introduce some more notation 
\begin{definition}
Given a sequence $\{\varphi_k\}_{k=0}^\infty$, we say that $\varphi_k$ converges to $\varphi \in \mathcal{D}'(\mathbb{R}^n)$ of $\varphi_k \to \varphi$ if
\begin{enumerate}
    \item There is a bounded set $M$ so that $\varphi_k$ vanishes outside $M$ for every $k$.
    \item For every multi-index $\alpha = (\alpha_1, \dots \alpha_n)$, we have that $D^\alpha\varphi_k \to  D^\alpha\varphi$ uniformly. That is, 
    \begin{equation}
        \frac{\partial^{\alpha_1}}{\partial x_1}\frac{\partial^{\alpha_2}}{\partial x_2}\dots \frac{\partial^{\alpha_n}}{\partial x_n}\varphi_k \to \frac{\partial^{\alpha_1}}{\partial x_1}\frac{\partial^{\alpha_2}}{\partial x_2}\dots \frac{\partial^{\alpha_n}}{\partial x_n}\varphi 
    \end{equation}
\end{enumerate}
\end{definition}
And now we use the standard sequential continuity criterion to define continuity on $\mathcal{D}'(\mathbb{R}^n)$ (or really for any Hilbert space):
\begin{definition}
A distribution $\zeta \in \mathcal{D}'(\mathbb{R}^n)$ is \textbf{continuous} if for every convergent sequence $\varphi_k \to \varphi$ in $C_c^\infty(\mathbb{R}^n)$, its image $\zeta(\varphi_k)$ converges to $\zeta(\varphi)$ in $\mathbb{C}$. 
\end{definition} 
Notice that this enforces a topological vector space as well as a continuous dual vector space structure. Let us explore a simple but important.

\begin{example}
Let $f: \mathbb{R}^n\to \mathbb{R}$ be such that for every compact set $K \subset \mathbb{R}^n$, 
\begin{equation*}
    \int_K|f| < \infty
\end{equation*}
(this property is called \textit{locally integrablity}) Then the linear functional $T_f: C_c^\infty(\mathbb{R}^n) \to \mathbb{R}$ given by
\begin{equation}
    \varphi \mapsto \int_{\mathbb{R}^n}f(x)\varphi(x)dx
\end{equation}
is a distribution. 
\end{example} 
\begin{example}
Using the above example, we notice because the Dirac delta "function"
\[   \delta (x) = \left\{
\begin{array}{ll}
      \infty & x = 0 \\
      0 & x \neq 0
\end{array} 
\right. \]
is locally integrable, it induces a distribution $\delta: C_c^\infty(\mathbb{R}^n) \to \mathbb{C}$ given by the linear map 
\begin{equation}
    \delta(\varphi) = \int_{\mathbb{R}^n}\delta(x)\varphi(x)dx = \varphi(0)
\end{equation}
\end{example}
Distributions have the attractive property that they can always be differentiated and for this reason, they are sometimes called \textit{generalized functions}. However, they are not evidently even remotely related to functions in the traditional sense. Perhaps most egregiously, one cannot even evaluate a distribution at a point $x \in \mathbb{R}^n$. Nonetheless, we can still speak about the value on open sets and in relation to "normal" functions: 
\begin{definition}
A distribution $\varphi: C_c^\infty(\mathbb{R}^n) \to \mathbb{C}$ \textbf{vanishes on an open set} $U \subset \mathbb{R}^n$ if $\varphi(f) = 0$ for every $f: \mathbb{R}^n \to \mathbb{C}$ supported inside $U$. \\\\
$\varphi$ \textbf{agrees with $f$} on an open set $U$ if 
\begin{equation}
    \varphi(\zeta) = \int_{\mathbb{R}^n}f(x)\zeta(x)dx
\end{equation}
for every $\zeta: \mathbb{R}^n \to \mathbb{C}$ supported in $U$.
\end{definition}
Furthermore, we can develop of a theory of differentiation for distributions which agrees with the usual theory for functions. To see how, first we will adopt the notational change for distributions $u(\varphi) \coloneqq \langle u, \varphi \rangle$.  
\begin{proposition}
Let $f:\mathbb{R}^n \to \mathbb{C}$ be continuously differentiable and denote $\partial^i f = \partial f /\partial x_i$. Then
\begin{equation}
    T_{\partial^i f}(\varphi) = - T_{f}(\partial^i\varphi)
\end{equation}
\end{proposition}
\begin{proof}
Because $\partial^i f$ is continuous, it defines a distribution $T_{\partial^i f} \in \mathcal{D}'(\mathbb{R}^n)$, given by 
\begin{equation*}
    T_{\partial^i f}(\varphi) = \int_{\mathbb{R}^n} \partial^i f(x)\varphi(x)dx
\end{equation*}
Using the integration by parts formula in multiple dimensions,
\begin{equation}
    \int_{\mathbb{R}^n}\partial^i f(x)\varphi(x)dx = - \int_{\mathbb{R}^n} f(x)\partial^i\varphi(x)dx + \int_{\mathbb{R}^n}\partial^i (f(x)\varphi(x))dx
\end{equation}
The rightmost term vanishes because of the fundamental theorem of calculus and that $\varphi$ is supported on a compact set. Hence,
\begin{equation}
    \int_{\mathbb{R}^n}\partial^i f(x)\varphi(x)dx = - \int_{\mathbb{R}^n} f(x)\partial^i\varphi(x)dx = -T_f(\partial^i\varphi)
\end{equation}
as desired. 
\end{proof}
The above proposition motivates the following notion of differentiability for distributions:
\begin{definition}
Let $\varphi \in \mathcal{D}'(\mathbb{R}^n)$. The \textbf{partial derivative in the $j$'th direction} is a distribution, denoted $\partial^i\varphi$ and is given by
\begin{equation}
    \partial^i\varphi(f) = \langle \partial^i\varphi, f \rangle = - \langle \varphi, \partial^if \rangle 
\end{equation}
for $f:\mathbb{R}^n \to \mathbb{C}$. Moreover, if $\alpha = (\alpha_1,\dots,\alpha_n)$ is any multi-index, then the distribution $\partial^\alpha$ is given by 
\begin{equation}
    \partial^\alpha\varphi(f) = \langle \partial^\alpha\varphi, f \rangle = (-1)^{|\alpha|} \langle \varphi, \partial^\alpha f \rangle 
\end{equation}
\end{definition}
Notice that in the theory of vector spaces, the operator $-\partial^i$ is just the adjoint of $\partial^i$
\begin{example}
Consider the Dirac $\delta$ distribution. We can compute its derivative in any particular direction $i$ as 
\begin{equation*}
    \langle \partial^i\delta, f \rangle =  -\langle \delta, \partial^i f \rangle = \frac{\partial f}{\partial x_i}(0)
\end{equation*}
\end{example}

Recall that a \textbf{Banach space} is a vector space $\mathcal{A}$ equipped with a norm $\|\cdot\|: \mathcal{A} \to \mathbb{R}$ such that the induced metric $d(x,y) = \|x-y\|$ is complete. A \textbf{Hilbert space} is a vector space with an inner product which is complete with respect to the induced metric $d(x,y) = \|x-y\| = \sqrt{\langle x-y,x-y\rangle}$. While abstract, the study of these spaces are undeniably important, and illuminate properties about important classes of functions. One such class (which will be extremely relevant to the paper) are the $L^p$ functions. 

\begin{definition}
The $L^p$ space $L^p(\mathbb{R}^d)$ is the real vector space of all functions whose $p$-th power is integrable for $p \in \mathbb{N}$. That is, all functions $f: \mathbb{R}^d\to \mathbb{R}$ with
\begin{equation*}
    \int_{\mathbb{R}^d}\lvert f \rvert^p< \infty
\end{equation*}
If $f: \mathbb{R}^d\to \mathbb{R}^k$, then we define the space $L^p(\mathbb{R}^d;\mathbb{R}^k)$ to be the space of functions $f$ with $\Vert f \Vert_{\mathbb{R}^k} \in L^p(\mathbb{R}^d)$.
\end{definition}
Notice that forms $L^p$ forms a Banach space by equipping the norm
\begin{equation*}
    \Vert f \Vert_{L^p} = \left(\int_{\mathbb{R}^d}\lvert f \rvert^p\right)^{1/p}
\end{equation*}
and that $L^p$ is a Hilbert space if and only if $p=2$. For the following definition, we recall that a \textbf{multi-index} is an $n$-tuple $\alpha = (\alpha_1,\alpha_2,\dots\alpha_n) \in \mathbb{N}^n$. To express strings of partial derivatives concisely, we write 
\begin{equation*}
    \partial^\alpha f \coloneqq \frac{\partial^{\alpha_1}}{\partial x_1}\frac{\partial^{\alpha_2}}{\partial x_2}\dots \frac{\partial^{\alpha_n}}{\partial x_n}f
\end{equation*}
with $\alpha$ a multi-index and $f: \mathbb{R}^n \to \mathbb{R}^d$.
\begin{definition}
The \textbf{Sobolev space} $W^{s,p}(\mathbb{R}^d)$ is the subspace of $L^p(\mathbb{R}^d)$ such that for every multi-index $\alpha$ with $\lvert\alpha\rvert\leq s$, $\partial^\alpha f \in L^p$ whenever $f \in L^p$. In this essay, we will only consider the spaces $W^{s,2}(\mathbb{R}^d)$, because they are Hilbert spaces for all $s$. In this case, we use the traditional notation $H^s(\mathbb{R}^d) \coloneqq W^{s,2}(\mathbb{R}^d)$. The space $H^s(\mathbb{R}^d; \mathbb{R}^k)$ is defined in a similar way as above. 
\end{definition}
We remark that the use of $\partial(\cdot)$ is formally taken in a "weak sense", but is a nuance that is largely removed from the purposes of this paper. 
\subsection{Linear Differential Operators and Green's Functions}
The following two subsections are dedicated to review some DE theory, as well as help motivate the construction of FNO's in the following chapter. \\\\
Recall that a \textbf{linear differential operator} is a linear map between function spaces $L: \mathcal{A} \to \mathcal{B}$ of the form: 
\begin{equation}
    L f = \sum_{\lvert\alpha\rvert\leq m} \varphi_{\alpha}\partial^\alpha f
\end{equation}
with $f, \varphi_\alpha: \mathbb{R}^n \to \mathbb{R}^d$ and $m$ a positive integer. Here, $\lvert\alpha\rvert = \alpha_1+\alpha_2+\alpha_3+\dots+\alpha_n$. The theory of differential equations often concern equations of the form 
\begin{equation}
    Lu(x) = f(x)
\end{equation}
Here $f: \mathbb{R}^n \to \mathbb{R}^d$ is given, and one tries to solve for $u$. One intuitive and popular technique to do so involves finding the \textit{Green's function}. 
\begin{definition}
The \textbf{Green's Function} $G$ of a linear differential operator $L$ satisfies
\begin{equation}\label{eq:a1}
    LG(x,y) = \delta(x-y)
\end{equation}
\end{definition}
This is useful, because given the Greens function $G$, we can compute a $u$  satisfying \eqref{eq:a1} for every $f$:
\begin{equation}
    u(x) = \int G(x,y)f(y)dy
\end{equation}
Indeed, 
\begin{align*}
    Lu(x) &= L\int G(x,y)f(y)dy \\
    &= \int LG(x,y)f(y)dy \\
    &= \int \delta(x-y)f(y)dy \\
    &= f(x)
\end{align*}
One may worry about the existence of Green's functions, but we will set those concerns aside with the following theorem \cite{10.2307/2372662}\cite{AIF_1956__6__271_0}.
\begin{theorem}[Malgrange-Ehrenpreis]
Every nonzero linear differential operator $L$ has a Green's Function (usually not unique!)
\end{theorem}
Finally, we remark that $G(x,y)$ can actually be written as a function of $x-y$ because $G(x+a,y+a) = G(x,y)$ for every $a$. We will henceforth use $G(x-y)$ to denote a Green's function (for reasons which will become apparent soon). None
\subsection{The Fourier Transform and Fast Fourier Transform}
Recall that the convolution between two continuous functions $f,g: \mathbb{R} \to \mathbb{R}$ is given by 
\begin{equation}
    (f*g)(x) = \int_\mathbb{R} f(y)g(x-y)dy
\end{equation}
And that the Fourier Transform of a (Lebesgue) integrable function $f: \mathbb{R} \to \mathbb{C}$ is given by 
\begin{equation}
    \mathcal{F}(f) = \int_\mathbb{R}f(x)e^{-2\pi i \xi x}dx
\end{equation}
for fixed $\xi \in \mathbb{R}$
and its inverse
\begin{equation}
    \mathcal{F}^{-1}(f) = \int_\mathbb{R}f(\xi)e^{2\pi i \xi x}d\xi
\end{equation}
for fixed $x \in \mathbb{R}$. When $f$ is instead defined on a finite space, the integral becomes simply a sum, and we obtain the discrete Fourier (and inverse) transform. 
The following is a relevant property relating the convolution and the product of Fourier transform.
\begin{proposition}[Convolution Theorem]
Let $f,g: \mathbb{R} \to \mathbb{R}$ be continuous. Then 
\begin{equation}
    (f*g)(x) = \mathcal{F}^{-1}(\mathcal{F}(f)\cdot\mathcal{F}(g))
\end{equation}
\end{proposition}
\begin{proof}
This follows quickly by applying definitions. Let us compute the Fourier transform of the convolution
\begin{align*}
    \mathcal{F}(f*g) &= \int \left(\int f(y)g(x-y)dy\right) e^{-2 \pi i \xi x}dx \\
    &= \int \left(\int g(x-y)e^{-2 \pi i \xi x}dx\right)f(y) dy
\end{align*}
where we may exchange integrals via Fubini's Theorem. 
Then making the substitution $u = x-y$, we have that 
\begin{align*}
    \int \left(\int g(x-y)e^{-2 \pi i \xi x}dx\right)f(y) dy &=\int \left(\int g(u)e^{-2 \pi i \xi (u+y)}du\right)f(y) dy \\
    &= \int \left(\mathcal{F}(g) e^{-2\pi i \xi y}\right)f(y) dy\\
    &= \mathcal{F}(f)\mathcal{F}(g)
\end{align*}
The statement of the theorem follows by applying the inverse transform on both sides.
\end{proof}
Now recall the general solution to the equation $Lu(x) = f(x)$ is given in terms of a Green's Function $G$ as.
\begin{equation*}
    u(x) = \int G(x-y)f(y)dy
\end{equation*}
However, a closed form usually is not known for most operators $L$. Hence, we may want to approximate $G$ numerically using a kernel, say $K(x-y) \approx G(x-y)$ and so 
\begin{equation}
    u(x) \approx \tilde{u}(x) \coloneqq \int K(x-y)f(y)dy = (K*f)(x)
\end{equation}
Notice that the equation is convolution form, so we may invoke the convolution theorem to yield 
\begin{equation}
    \tilde{u}(x) = \mathcal{F}^{-1}(\mathcal{F}(K)\cdot\mathcal{F}(f))
\end{equation}
Notice that $\mathcal{F}(f)$ is known, so our objective is to learn the kernel $K$ (in practice, $K$ is often given a certain parameter space to learn from: $K(x-y) \coloneqq K(x-y; \theta)$.




\end{appendices}

\bibliography{references.bib}

\end{document}